\documentclass{article} 
\usepackage{nips15submit_e,times}
\usepackage{url}

\usepackage{amsthm}

\newtheorem{rmk}{Definition}
\usepackage{enumerate}

\usepackage{framed,multirow}
\usepackage{graphicx}
\usepackage{graphics}

\usepackage{amssymb}
\usepackage{latexsym}
\usepackage{xcolor,colortbl}
\usepackage{amsmath}

\definecolor{Gray}{gray}{0.85}
\definecolor{LightCyan}{rgb}{0.88,1,1}
\definecolor{DarkYellow}{rgb}{0.88,1,1}
\definecolor{LightYellow}{rgb}{0.88,1,1}
\definecolor{chartreuse(traditional)}{rgb}{0.87, 1.0, 0.0}

\newcolumntype{a}{>{\columncolor{Gray}}c}
\newcolumntype{b}{>{\columncolor{white}}c}

\newcolumntype{a}{>{\columncolor{red}}c}
\newcolumntype{b}{>{\columncolor{orange}}c}
\newcolumntype{d}{>{\columncolor{yellow}}c}
\newcolumntype{e}{>{\columncolor{chartreuse(traditional)}}c}
\newcolumntype{f}{>{\columncolor{green}}c}

\usepackage{url}
\usepackage{xcolor}
\definecolor{newcolor}{rgb}{.8,.349,.1}

\usepackage{natbib}

\title{Toward a generic representation of random variables for machine learning}

\author{
Gautier Marti \\
Hellebore Capital Management\\
63, Avenue des Champs-Elys\'ees\\
Paris, 75008 \\
\texttt{gmarti@helleborecapital.com} \\
\And
Philippe Very \\
Hellebore Capital Management \\
63, Avenue des Champs-Elys\'ees\\
Paris, 75008 \\
\texttt{pvery@helleborecapital.com} \\
\AND
Philippe Donnat \\
Hellebore Capital Management \\
63, Avenue des Champs-Elys\'ees\\
Paris, 75008 \\
\texttt{pdonnat@helleborecapital.com} \\
}

\nipsfinalcopy 

\begin{document}

\maketitle

\begin{abstract}
This paper presents a pre-processing and a distance which improve the performance of machine learning algorithms working on independent and identically distributed stochastic processes.
We introduce a novel non-parametric approach to represent random variables which splits apart dependency and distribution without losing any information.
We also propound an associated metric leveraging this representation and its statistical estimate.
Besides experiments on synthetic datasets, the benefits of our contribution is illustrated through the example of clustering financial time series, for instance prices from the credit default swaps market. Results are available on the website \url{www.datagrapple.com} and an IPython Notebook tutorial is available at \url{www.datagrapple.com/Tech} for reproducible research.
\end{abstract}

\section{Introduction}
Machine learning on time series is a booming field and as such plenty of representations, transformations, normalizations, metrics and other divergences are thrown at disposal to the practitioner.
A further consequence of the recent advances in time series mining is that it is difficult to have a sober look at the state of the art since many papers state contradictory claims as described in \citep{exp_keogh}.
To be fair, we should mention that when data, pre-processing steps, distances and algorithms are combined together, they have an intricate behaviour making it difficult to draw unanimous conclusions especially in a fast-paced environment.
Restricting the scope of time series to independent and identically distributed (i.i.d.) stochastic processes, we propound a method which, on the contrary to many of its counterparts, is mathematically grounded with respect to the clustering task defined in subsection~\ref{motivation}. The representation we present in Section \ref{copula_repr_section} exploits a property similar to the seminal result of copula theory, namely Sklar's theorem \citep{sklar_theo}. This approach leverages the specificities of random variables and this way solves several shortcomings of more classical data pre-processing and distances that will be detailed in subsection \ref{shortcomings_section}.
Section \ref{exp_section} is dedicated to experiments on synthetic and real datasets to illustrate the benefits of our method which relies on the hypothesis of i.i.d. sampling of the random variables. 
Synthetic time series are generated by a simple model yielding correlated random variables following different distributions.
The presented approach is also applied to financial time series from the credit default swaps market whose prices dynamics are usually modelled by random walks according to the efficient-market hypothesis \citep{citeulike:1571390}. This dataset seems more interesting than stocks as credit default swaps are often considered as a gauge of investors' fear, thus time series are subject to more violent moves and may provide more distributional information than the ones from the stock market.
We have made our detailed experiments (cf. Machine Tree on the website \url{www.datagrapple.com}) and Python code available (\url{www.datagrapple.com/Tech}) for reproducible research.
Finally, we conclude the paper with a discussion on the method and we propound future research directions.


\subsection{Motivation and goal of study}\label{motivation}

Machine learning methodology usually consists in several pre-processing steps aiming at cleaning data and preparing them for being fed to a battery of algorithms.
Data scientists have the daunting mission to choose the best possible combination of pre-processing, dissimilarity measure and algorithm to solve the task at hand among a profuse literature. In this article, we provide both a pre-processing and a distance for studying i.i.d. random processes which are compatible with basic machine learning algorithms.

Many statistical distances exist to measure the dissimilarity of two random variables, and therefore two i.i.d. random processes. Such distances can be roughly classified in two families:
\begin{enumerate}
\item distributional distances, for instance \citep{clust_proc}, \citep{KhaleghiRMP12} and \citep{henderson2015ep}, which focus on dissimilarity between probability distributions and quantify divergences in marginal behaviours,  
\item dependence distances, such as the distance correlation or copula-based kernel dependency measures \citep{copula_kernel}, which focus on the joint behaviours of random variables, generally ignoring their distribution properties.
\end{enumerate}
However, we may want to be able to discriminate random variables both on distribution and dependence. This can be motivated, for instance, from the study of financial assets returns: are two perfectly correlated random variables (assets returns), but one being normally distributed and the other one following a heavy-tailed distribution, similar? From risk perspective, the answer is no \citep{kelly2014tail}, hence the propounded distance of this article.
We illustrate its benefits through clustering, a machine learning task which primarily relies on the metric space considered (data representation and associated distance). Besides clustering has found application in finance, e.g. \citep{tola2008cluster}, which gives us a framework for benchmarking on real data.

Our objective is therefore to obtain a good clustering of random variables based on an appropriate and simple enough distance for being used with basic clustering algorithms, e.g. Ward hierarchical clustering \citep{Inchoate:Ward63}, $k$-means++ \citep{kmeanspp}, affinity propagation \citep{frey2007clustering}.

By clustering we mean the task of grouping sets of objects in such a way that objects in the same cluster are more similar to each other than those in different clusters. More specifically, a cluster of random variables should gather random variables with common dependence between them and with a common distribution. Two clusters should differ either in the dependency between their random variables or in their distributions.

A good clustering is a partition of the data that must be stable to small perturbations of the dataset. ``Stability of some kind is clearly a desirable property of clustering methods" \citep{ultrametric_hc}.
In the case of random variables, these small perturbations can be obtained from resampling \citep{levine2001resampling}, \citep{monti2003consensus}, \citep{lange2004stability} in the spirit of the bootstrap method since it preserves the statistical properties of the initial sample \citep{bootstrap_efron}.

Yet, practitioners and researchers pinpoint that state-of-the-art results of clustering methodology applied to financial times series are very sensitive to perturbations \citep{lemieux2014clustering}. The observed unstability may result from a poor representation of these time series, and thus clusters may not capture all the underlying information.

\subsection{Shortcomings of a standard machine learning approach} \label{shortcomings_section}

A naive but often used distance between random variables to measure similarity and to perform clustering is the $L_2$ distance $\mathbb{E}[ (X-Y)^2 ]$.
Yet, this distance is not suited to our task. 
\newtheorem{exmp}{Example}
\begin{exmp}[Distance $L_2$ between two Gaussians]
Let $(X,Y)$ be a bivariate Gaussian vector, with $X \sim \mathcal{N}(\mu_X,\sigma_X^2)$, $Y \sim \mathcal{N}(\mu_Y,\sigma_Y^2)$ and whose correlation is $\rho(X,Y) \in [-1,1]$. We obtain $\mathbb{E}[ (X-Y)^2 ] = (\mu_X - \mu_Y)^2 + (\sigma_X - \sigma_Y)^2 + 2\sigma_X \sigma_Y ( 1 - \rho(X,Y) )$. Now, consider the following values for correlation:
\begin{itemize}
\item $\rho(X,Y) = 0$, so $\mathbb{E}[ (X-Y)^2 ] = (\mu_X - \mu_Y)^2 + \sigma_X^2 + \sigma_Y^2$.
The two variables are independent (since uncorrelated and jointly normally distributed), thus we must discriminate on distribution information. Assume $\mu_X = \mu_Y$ and $\sigma_X = \sigma_Y$. For $\sigma_X = \sigma_Y \gg 1$, we obtain $\mathbb{E}[ (X-Y)^2 ] \gg 1$ instead of the distance $0$, expected from comparing two equal Gaussians.
\item $\rho(X,Y) = 1$, so $\mathbb{E}[ (X-Y)^2 ] = (\mu_X - \mu_Y)^2 + (\sigma_X - \sigma_Y)^2$.
Since the variables are perfectly correlated, we must discriminate on distributions. We actually compare them with a $L_2$ metric on the mean $\times$ standard deviation half-plane. However, this is not an appropriate geometry for comparing two Gaussians \citep{costa2014fisher}. For instance, if $\sigma_X = \sigma_Y = \sigma$, we find $\mathbb{E}[ (X-Y)^2 ] = (\mu_X - \mu_Y)^2$ for any values of $\sigma$. As $\sigma$ grows, probability attached by the two Gaussians to a given interval grows similar (cf. Fig.~\ref{fig:equidist_Gaussians}), yet this increasing similarity is not taken into account by this $L_2$ distance.
\begin{figure}
\vskip 0.2in
\begin{center}
\includegraphics[width=0.9\linewidth]{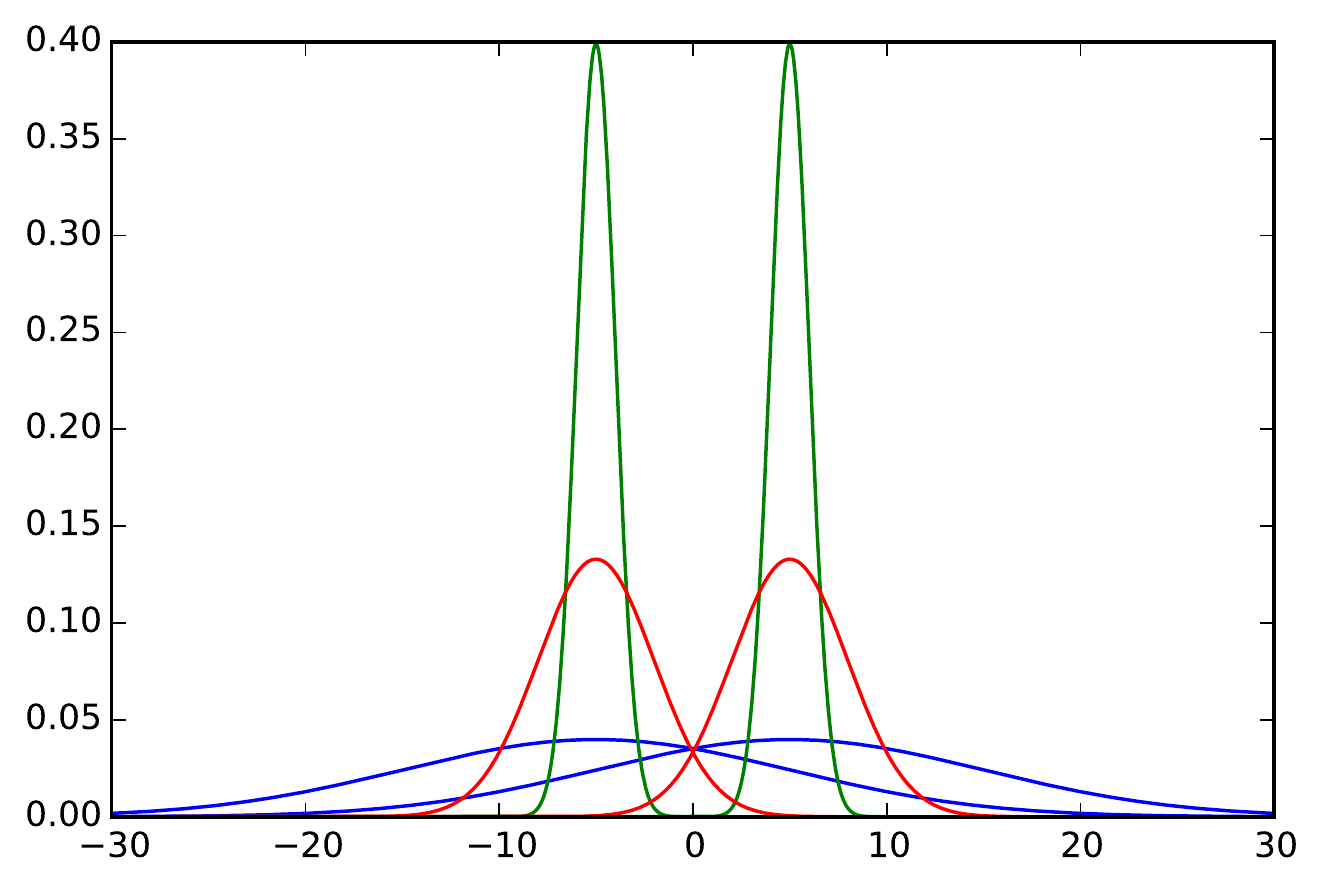}
\caption{Probability density functions of Gaussians $\mathcal{N}(-5,1)$ and $\mathcal{N}(5,1)$ (in green), Gaussians $\mathcal{N}(-5,3)$ and $\mathcal{N}(5,3)$ (in red), and Gaussians $\mathcal{N}(-5,10)$ and $\mathcal{N}(5,10)$ (in blue). Green, red and blue Gaussians are equidistant using $L_2$ geometry on the parameter space $(\mu,\sigma)$.}
\label{fig:equidist_Gaussians}
\end{center}
\vskip -0.2in
\end{figure}
\end{itemize}
\end{exmp}
$\mathbb{E}[(X-Y)^2]$ considers both dependence and distribution information of the random variables, but not in a relevant way with respect to our task. Yet, we will benchmark against this distance since other more sophisticated distances on time series such as dynamic time warping \citep{berndt1994using} and representations such as wavelets \citep{percival2006wavelet} or SAX \citep{lin2003symbolic} were explicitly designed to handle temporal patterns which are inexistant in i.i.d. random processes.

\section{A generic representation for random variables}\label{copula_repr_section}

Our purpose is to introduce a new data representation and a suitable distance which takes into account both distributional proximities and joint behaviours.

\newtheorem{thm}{Theorem}
\newtheorem{prop}[thm]{Property}

\subsection{A representation preserving total information}
Let $(\Omega,\mathcal{F},\mathbb{P})$ be a probability space. $\Omega$ is the sample space, $\mathcal{F}$ is the $\sigma$-algebra of events, and $\mathbb{P}$ is the probability measure.
Let $\mathcal{V}$ be the space of all continuous real-valued random variables defined on $(\Omega,\mathcal{F},\mathbb{P})$.
Let $\mathcal{U}$ be the space of random variables following a uniform distribution on $[0,1]$ and $\mathcal{G}$ be the space of absolutely continuous cumulative distribution functions (cdf).
\begin{rmk}[The copula transform]
Let $X = (X_1,\ldots,X_N) \in \mathcal{V}^N$ be a random vector with cdfs $G_X = (G_{X_1},\ldots,G_{X_N}) \in \mathcal{G}^N$. The random vector $G_X(X) = (G_{X_1}(X_1),\ldots,G_{X_N}(X_N)) \in \mathcal{U}^N$ is known as the copula transform.
\end{rmk}
\begin{prop}[Uniform margins of the copula transform]
$G_{X_i}(X_i)$, $1 \leq i \leq N$, are uniformly distributed on $[0,1]$.
\end{prop}
\begin{proof}
$x = G_{X_i}(G_{X_i}^{-1}(x)) = \mathbb{P}(X_i \leq G_{X_i}^{-1}(x)) = \mathbb{P}(G_{X_i}(X_i) \leq x)$.
\end{proof}
We define the following representation of random vectors  that actually splits the joint behaviours of the marginal variables from their distributional information.
\begin{rmk}[dependence $\oplus$ distribution space projection]\label{GNPR}
Let $\mathcal{T}$ be a mapping which transforms $X = (X_1,\ldots,X_N)$ into its generic representation, an element of $ \mathcal{U}^N\times\mathcal{G}^N$ representing $X$, defined as follow
\begin{eqnarray}
\mathcal{T}:\mathcal{V}^N & \rightarrow & \mathcal{U}^N\times\mathcal{G}^N \\ 
X & \mapsto & (G_X(X),G_X). \nonumber
\end{eqnarray}
\end{rmk}
\begin{prop}\label{bij}
$\mathcal{T}$ is a bijection.
\end{prop}
\begin{proof}
$\mathcal{T}$ is surjective as any element $(U,G) \in \mathcal{U}^N\times\mathcal{G}^N$ has the fiber $G^{-1}(U)$. $\mathcal{T}$ is injective as $(U_{1},G_{1}) = (U_{2},G_{2})$ \textit{a.s.} in  $\mathcal{U}^N\times\mathcal{G}^N$ implies that they have the same cdf $G=G_{1}=G_{2}$ and since $U_{1}=U_{2}$ \textit{a.s.}, it follows that $G^{-1}(U_{1})=G^{-1}(U_{2})$ \textit{a.s.}
\end{proof}
This result replicates the seminal result of copula theory, namely Sklar's theorem \citep{sklar_theo}, which asserts one can split the dependency and distribution apart without losing any information. Fig.~\ref{GNPR_projection} illustrates this projection for $N = 2$.

\begin{figure*}
\vskip 0.2in
\begin{center}
\centerline{\includegraphics[width=0.3\textwidth]{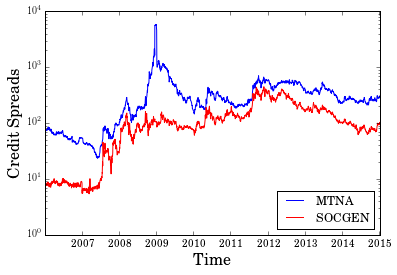}
\includegraphics[width=0.03\textwidth]{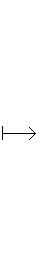}
\includegraphics[width=0.3\textwidth]{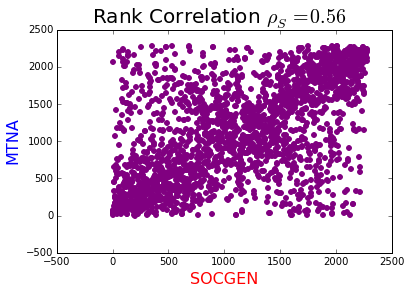} \includegraphics[width=0.03\textwidth]{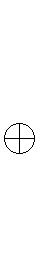} 
\includegraphics[width=0.29\textwidth]{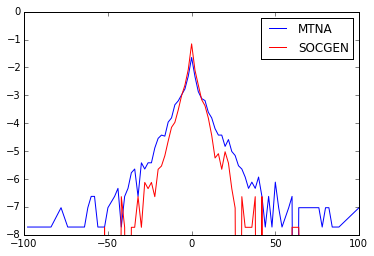}
}
\caption{ArcelorMittal and Soci\'et\'e g\'en\'erale prices ($T$ observations  $(X_1^t,X_2^t)_{t=1}^T$ from $(X_1,X_2) \in \mathcal{V}^2$) are projected on dependence $\oplus$ distribution space; $(G_{X_1}(X_1),G_{X_2}(X_2)) \in \mathcal{U}^2$ encode the dependence between $X_1$ and $X_2$ (a perfect correlation would be represented by a sharp diagonal on the scatterplot); $(G_{X_1},G_{X_2})$ are the margins (their log-densities are displayed above), notice their heavy-tailed exponential distribution (especially for ArcelorMittal).}
\label{GNPR_projection}
\end{center}
\vskip -0.2in
\end{figure*}

\subsection{A distance between random variables}

We leverage the propounded representation to build a suitable yet simple distance between random variables which is invariant under diffeomorphism.

\begin{rmk}[Distance $d_\theta$ between two random variables] Let $\theta \in [0,1]$. Let $(X,Y) \in \mathcal{V}^2$. Let $G = (G_{X},G_{Y})$, where $G_X$ and $G_Y$ are respectively $X$ and $Y$ marginal cdfs. We define the following distance
\begin{eqnarray}\label{distance_theta}
d_{\theta}^{2}(X,Y)=\theta d_{1}^{2}(G_{X}(X),G_{Y}(Y))+(1-\theta)d_{0}^{2}(G_{X},G_{Y}),
\end{eqnarray}
where
\begin{eqnarray}
d_{1}^{2}(G_X(X),G_Y(Y))=3\mathbb{E}[\vert G_{X}(X)-G_{Y}(Y) \vert^{2}],
\end{eqnarray}
and
\begin{eqnarray}
d_{0}^{2}(G_{X},G_{Y})=\frac{1}{2}\int_{\mathbf{R}} \left(\sqrt{\frac{dG_{X}}{d\lambda}}-\sqrt{\frac{dG_{Y}}{d\lambda}}\right)^2 \, \mathrm{d}\lambda.
\end{eqnarray}
\end{rmk}
In particular, $d_0 = \sqrt{1 - BC}$ is the Hellinger distance related to the Bhattacharyya (1/2-Chernoff) coefficient $BC$ upper bounding the Bayes' classification error.
To quantify distribution dissimilarity, $d_0$ is used rather than the more general $\alpha$-Chernoff divergences since it satisfies the properties \ref{propa}, \ref{propb}, \ref{diffeo} (significant for practitioners). In addition, $d_\theta$ can thus be efficiently implemented as a scalar product.
$d_1 = \sqrt{(1 - \rho_S) / 2}$ is a distance correlation measuring statistical dependence between two random variables, where $\rho_S$ is the Spearman's correlation between $X$ and $Y$. Notice that $d_1$ can be expressed by using the copula $C : [0,1]^2 \rightarrow [0,1]$ implicitly defined by the relation $G(X,Y) = C(G_X(X),G_Y(Y))$ since $\rho_S(X,Y) = 12 \int_{0}^{1}\int_{0}^{1} C(u,v) ~\mathrm{d}u ~\mathrm{d}v - 3$ \citep{fredricks2007relationship}.
\begin{exmp}[Distance $d_\theta$ between two Gaussians]
Let $(X,Y)$ be a bivariate Gaussian vector, with $X \sim \mathcal{N}(\mu_X,\sigma_X^2)$, $Y \sim \mathcal{N}(\mu_Y,\sigma_Y^2)$ and $\rho(X,Y) = \rho$.
We obtain, 
$$d_\theta^2(X,Y) = \theta \frac{1 - \rho_S}{2} + (1-\theta) \left(1 - \sqrt{\frac{2 \sigma_X \sigma_Y}{\sigma_X^2 + \sigma_Y^2}}e^{-\frac{1}{4} \frac{(\mu_X - \mu_Y)^2}{\sigma_X^2 + \sigma_Y^2}}\right).$$
\end{exmp}
Remember that for perfectly correlated Gaussians ($\rho = \rho_S = 1$), we want to discriminate on their distributions. We can observe that
\begin{itemize}
\item for $\sigma_X,\sigma_Y \rightarrow +\infty$, then $d_0(X,Y) \rightarrow 0$, it alleviates a main shortcoming of the basic $L_2$ distance which is diverging to $+\infty$ in this case;
\item if $\mu_X \neq \mu_Y$, for $\sigma_X,\sigma_Y \rightarrow 0$, then $d_0(X,Y) \rightarrow 1$, its maximum value, i.e. it means that two Gaussians cannot be more remote from each other than two different Dirac delta functions.
\end{itemize}
We will refer to the use of this distance as the generic parametric representation (GPR) approach.
GPR distance is a fast and good proxy for distance $d_\theta$ when the first two moments $\mu$ and $\sigma$ predominate. Nonetheless, for datasets which contain heavy-tailed distributions, GPR fails to capture this information.

\begin{prop}\label{propa}
Let $\theta \in [0,1]$. The distance $d_\theta$ verifies $0 \leq d_\theta \leq 1$.
\end{prop}
\begin{proof} Let $\theta \in [0,1]$. We have
\begin{enumerate}[(i)]
\item $0 \leq d_0 \leq 1$, property of the Hellinger distance;
\item $0 \leq d_1 \leq 1$, since $-1 \leq \rho_S \leq 1$.
\end{enumerate}
Finally, by convex combination, $0 \leq d_{\theta} \leq 1$.
\end{proof}
\begin{prop}\label{propb}
For $0 < \theta < 1$, $d_\theta$ is a metric.
\end{prop}
\begin{proof} Let $(X,Y) \in \mathcal{V}^2$. For $0 < \theta < 1$, $d_\theta$ is a metric, and the only non-trivial property to verify is the separation axiom
\begin{enumerate}[(i)]
\item $X = Y$ \textit{a.s.} $\Rightarrow d_\theta(X,Y) = 0$
\\ $X = Y$ \textit{a.s.} $\Rightarrow d_1(G_X(X),G_Y(Y)) = d_0(G_X,G_Y) = 0$, and thus $d_\theta(X,Y) = 0$,
\item $d_\theta(X,Y) = 0 \Rightarrow X = Y$ \textit{a.s.} \\
$d_\theta(X,Y) = 0 \Rightarrow d_1(G_X(X),G_Y(Y)) = 0 $ and $d_0(G_X,G_Y) = 0$ $\Rightarrow G_X(X) = G_Y(Y)$ \textit{a.s.} and $G_X = G_Y$. Since $G$ is absolutely continuous, it follows $X = Y$ \textit{a.s.}
\end{enumerate}
Notice that for $\theta \in \{0,1\}$, this property does not hold.
Let $U \in \mathcal{V}$, $U \sim \mathcal{U}[0,1]$. $U \neq 1-U$ but $d_0(U,1-U) = 0$.
Let $V \in \mathcal{V}$. $V \neq 2V$ but $d_1(V,2V) = 0$.
\end{proof}
\begin{prop} \label{diffeo} Diffeomorphism invariance. Let $h : \mathcal{V} \rightarrow \mathcal{V}$ be a diffeomorphism. Let $(X,Y) \in \mathcal{V}^{2}$. Distance $d_{\theta}$ is invariant under diffeomorphism, i.e.
\begin{equation}
d_{\theta}(h(X),h(Y)) = d_{\theta}(X,Y).
\end{equation}
\end{prop}
\begin{proof}From definition, we have
\begin{equation}
d_{0}^{2}(h(X),h(Y))=1-\int_{\mathbf{R}}\sqrt{\frac{dG_{h(X)}}{d\lambda}\frac{dG_{h(Y)}}{d\lambda}}d\lambda
\end{equation}
and since
\begin{equation}
\frac{dG_{h(X)}}{d\lambda}(\lambda)=\frac{1}{h'\left(h^{-1}(\lambda) \right)}\frac{dG_{X}}{d\lambda}\left(h^{-1}(\lambda) \right),
\end{equation}
we obtain
\begin{equation} \label{proof_diffeo}
\begin{split}
d_{0}^{2}(h(X),h(Y)) & =
1-\int_{\mathbf{R}}\frac{1}{h'\left(h^{-1}(\lambda) \right)}\sqrt{\frac{dG_{X}}{d\lambda}\frac{dG_{Y}}{d\lambda}}\left(h^{-1}(\lambda) \right)d\lambda \\
& = d_{0}^{2}(X,Y).
\end{split}
\end{equation}
In addition, $\forall x \in \mathbf{R}$, we have
\begin{equation}
\begin{split}
G_{h(X)}\left(h(x)\right)& = \mathbb{P}\left[h(X)\le h(x)\right] \\
 & = \left\{
 		\begin{split} 
 		\mathbb{P}\left[X\le x\right] & = G_{X}(x),\mathrm{~if~}h\mathrm{~increasing} \\
 		1 - \mathbb{P}\left[X\le x\right] & = 1 - G_{X}(x), \mathrm{~otherwise}
 		\end{split}
 	\right. 
\end{split}
\end{equation}
which implies that 
\begin{equation} \label{expect_diffeo}
\begin{split}
d_{1}^{2}\left(h(X),h(Y)\right) & = 3\mathbb{E}\left[|G_{h(X)}(h(X))-G_{h(Y)}(h(Y))|^{2}\right] \\ 
& =3 \mathbb{E}\left[|G_{X}(X)-G_{Y}(Y)|^{2}\right] \\
& = d_{1}^{2}(X,Y).
\end{split}
\end{equation}
Finally, we obtain Property \ref{diffeo} by definition of $d_{\theta}$.
\end{proof}
Thus, $d_\theta$ is invariant under monotonic transformations, a desirable property
as it ensures to be insensitive to scaling (e.g. choice of units) or measurement scheme (e.g. device, mathematical modelling) of the underlying phenomenon.


\subsection{A non-parametric statistical estimation of $d_\theta$}

To apply the propounded distance $d_\theta$ on sampled data without parametric assumptions, we have to define its statistical estimate $\tilde{d}_\theta$ working on realizations of the i.i.d. random variables. Distance $d_1$ working with continuous uniform distributions can be approximated by normalized rank statistics yielding to discrete uniform distributions, in fact coordinates of the multivariate empirical copula \citep{deheuvels1979} which is a non-parametric estimate converging uniformly toward the underlying copula \citep{deheuvels1981}.
Distance $d_0$ working with densities can be approximated by using its discrete form working on histogram density estimates.

\begin{rmk}[The empirical copula transform]
Let $X^T = (X_1^t,\ldots,X_N^t)$, $t=1,\ldots,T$, be $T$ observations from a random vector $X = (X_1,\ldots,X_N)$ with continuous margins $G_X = (G_{X_1}(X_1),\ldots,G_{X_N}(X_N))$. Since one cannot directly obtain the corresponding copula observations $(G_{X_1}(X_1^t),\ldots,G_{X_N}(X_N^t))$ without knowing a priori $G_X$, one can instead estimate the $N$ empirical margins $G_{X_i}^T(x) = \frac{1}{T}\sum_{t=1}^T \mathbf{1}(X_i^t \leq x)$ to obtain $T$ empirical observations $(G_{X_1}^T(X_1^t),\ldots,G_{X_N}^T(X_N^t))$ which are thus related to normalized rank statistics as $G_{X_i}^T(X_i^t) = X_i^{(t)} / T$, where $X_i^{(t)}$ denotes the rank of observation $X_i^t$.
\end{rmk}
\begin{rmk}[Empirical distance]
Let $(X^t)_{t=1}^T$ and $(Y^t)_{t=1}^T$ be $T$ realizations of real-valued random variables $X, Y \in \mathcal{V}$ respectively. 
An empirical distance between realizations of random variables can be defined by
\begin{eqnarray}
\tilde{d}_\theta^2\left((X^t)_{t=1}^T,(Y^t)_{t=1}^T\right) \stackrel{a.s.}{=} \theta \tilde{d}_{1}^2 + (1-\theta) \tilde{d}_{0}^2,
\end{eqnarray}
where
\begin{eqnarray}
\tilde{d}_1^2 = \frac{3}{T(T^2-1)}\sum_{t=1}^T \left(X^{(t)} - Y^{(t)}\right)^2
\end{eqnarray}
and
\begin{eqnarray}
\tilde{d}_0^2=\frac{1}{2}\sum_{k=-\infty}^{+\infty} \left(\sqrt{g_X^h(hk)} - \sqrt{g_Y^h(hk)} \right)^2,
\end{eqnarray}
$h$ being here a suitable bandwidth, and $g_X^h(x) = \frac{1}{T}\sum_{t=1}^T\mathbf{1}( \lfloor \frac{x}{h} \rfloor h\leq X^t < (\lfloor \frac{x}{h} \rfloor +1)h )$ being a density histogram estimating pdf $g_X$ from $(X^t)_{t=1}^T$, $T$ realizations of random variable $X \in \mathcal{V}$.
\end{rmk}
We will refer henceforth to this distance and its use as the generic non-parametric representation (GNPR) approach. To use effectively $d_{\theta}$ and its statistical estimate, it boils down to select a particular value for $\theta$. We suggest here an exploratory approach where one can test (i) distribution information ($\theta = 0$), (ii) dependence information ($\theta = 1$), and (iii) a mix of both information ($\theta = 0.5$).
Ideally, $\theta$ should reflect the balance of dependence and distribution information in the data. 
In a supervised setting, one could select an estimate $\hat{\theta}$ of the right balance $\theta^\star$ optimizing some loss function by techniques such as cross-validation.
Yet, the lack of a clear loss function makes the estimation of $\theta^\star$ difficult in an unsupervised setting. For clustering, many authors \citep{lange2004stability}, \citep{shamir2007cluster}, \citep{shamir2008model}, \citep{meinshausen2010stability} suggest stability as a tool for parameter selection. But, \citep{ben2006sober} warn against its irrelevant use for this purpose. Besides, we already use stability for clustering validation and we want to avoid overfitting.
Finally, we think that finding an optimal trade-off $\theta^\star$ is important for accelerating the rate of convergence toward the underlying ground truth when working with finite and possibly small samples, but ultimately lose its importance asymptotically as soon as $0 < \theta < 1$.

\section{Experiments and applications}\label{exp_section}

\subsection{Synthetic datasets}

We propose the following model for testing the efficiency of the GNPR approach: $N$ time series of length $T$ which are subdivided into $K$ correlation clusters themselves subdivided into $D$ distribution clusters. 

Let $(Y_k)_{k=1}^{K}$, be $K$ i.i.d. random variables. 
Let $p,D \in \mathbf{N}$. Let $N = p K D$. Let $(Z_d^i)_{d=1}^{D}$, $1 \leq i \leq N$, be independent random variables. 
For $1 \leq i \leq N$, we define
\begin{eqnarray}
X_i = \sum_{k=1}^{K}\beta_{k,i} Y_k + \sum_{d=1}^{D}\alpha_{d,i}Z_d^i,
\end{eqnarray} 
where
\begin{enumerate}[a)]
\item $\alpha_{d,i} = 1$, if $i \equiv d-1$ (mod $D$), $0$ otherwise;
\item $\beta \in [0,1]$,
\item $\beta_{k,i} = \beta$, if $\lceil i K / N \rceil = k$, $0$ otherwise.
\end{enumerate}
$(X_i)_{i=1}^N$ are partitioned into $Q = K D$ clusters of $p$ random variables each.
Playing with the model parameters, we define in Table $\ref{tab2}$ some interesting test case datasets to study distribution clustering, dependence clustering or a mix of both.
We use the following notations as a shorthand:
$\mathcal{L} := \mathrm{Laplace}(0,1/\sqrt{2})$ and $\mathcal{S} := \mathrm{t}$-distribution$(3)/\sqrt{3}$.
Since $\mathcal{L}$ and $\mathcal{S}$ have both a mean of $0$ and a variance of $1$, GPR cannot find any difference between them, but GNPR can discriminate on higher moments
as it can be seen in Fig.~\ref{fig:dist_comp_GPR_GNPR}.
\begin{figure}[!t]
\begin{minipage}[b]{0.325\linewidth}
\centering
\includegraphics[width=\textwidth]{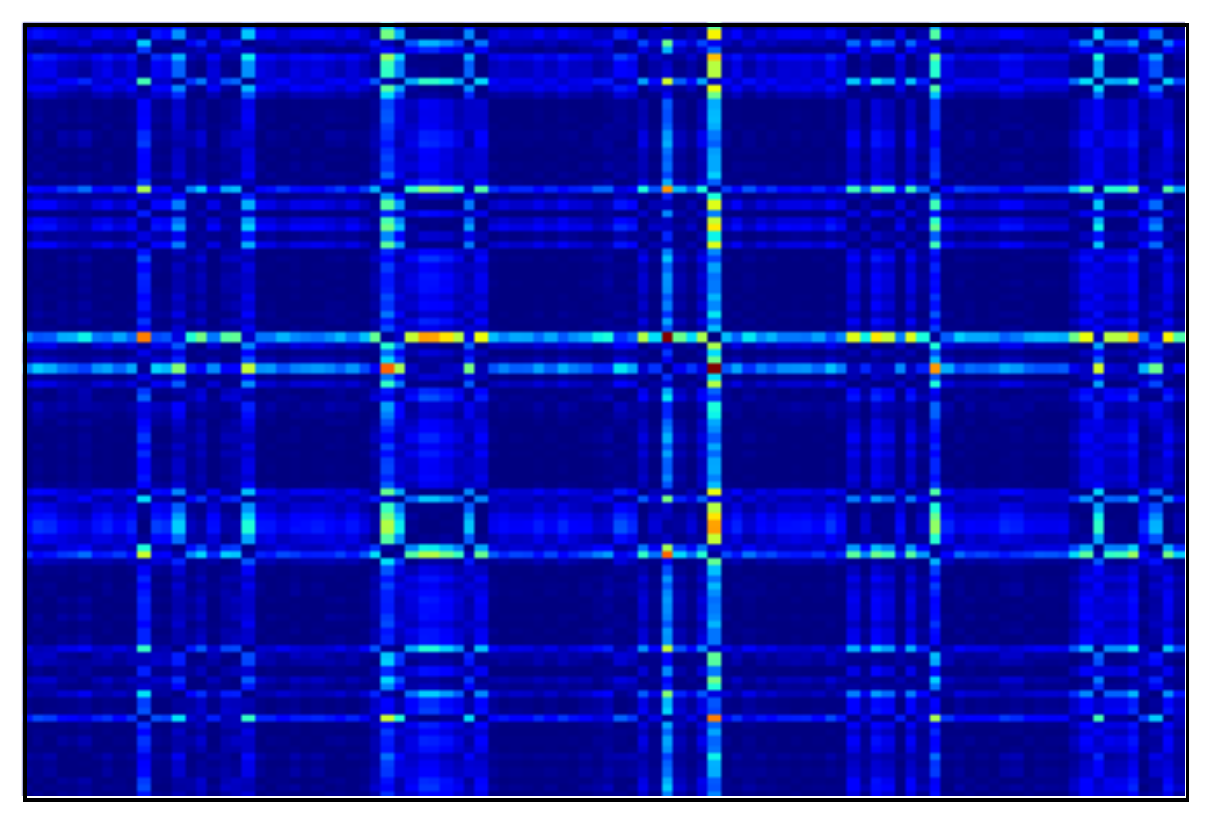}
GPR $\theta = 0$
\end{minipage}
\begin{minipage}[b]{0.325\linewidth}
\centering
\includegraphics[width=\textwidth]{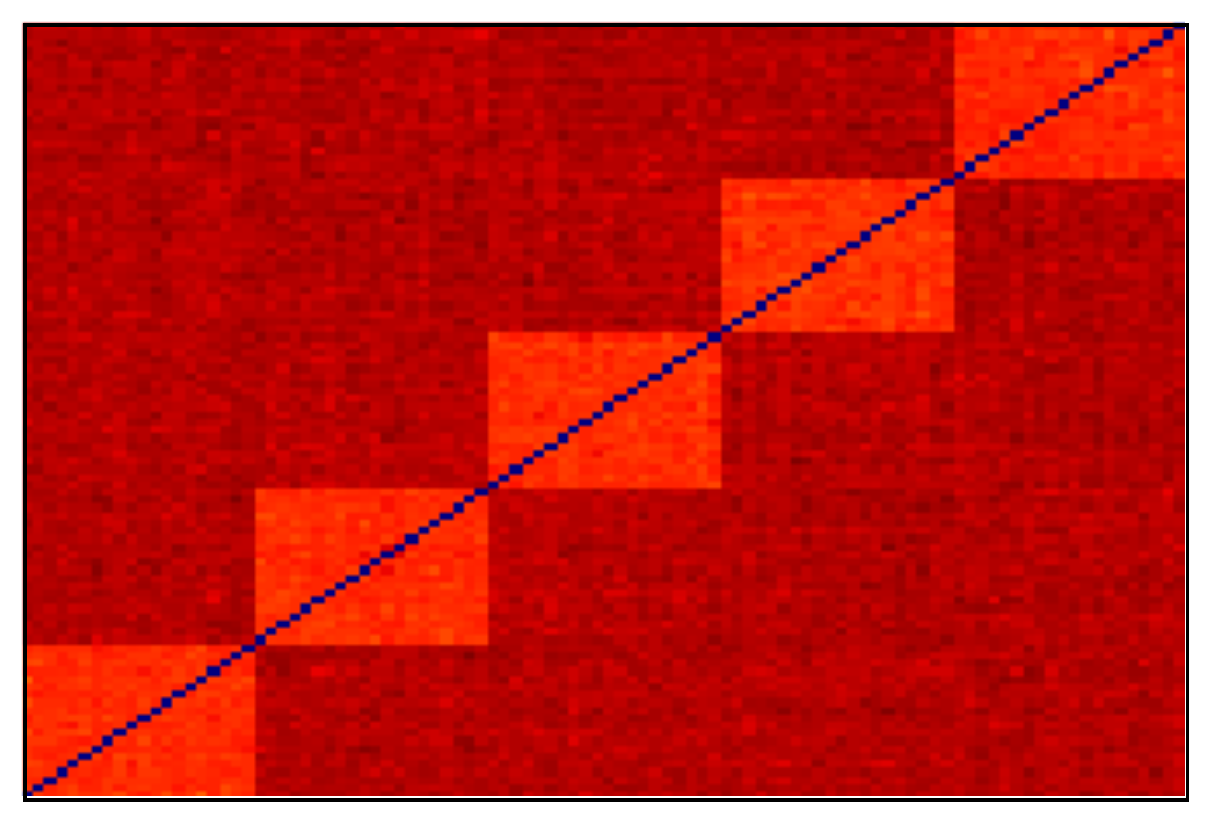}
GPR $\theta = 1$
\end{minipage}
\begin{minipage}[b]{0.325\linewidth}
\centering
\includegraphics[width=\textwidth]{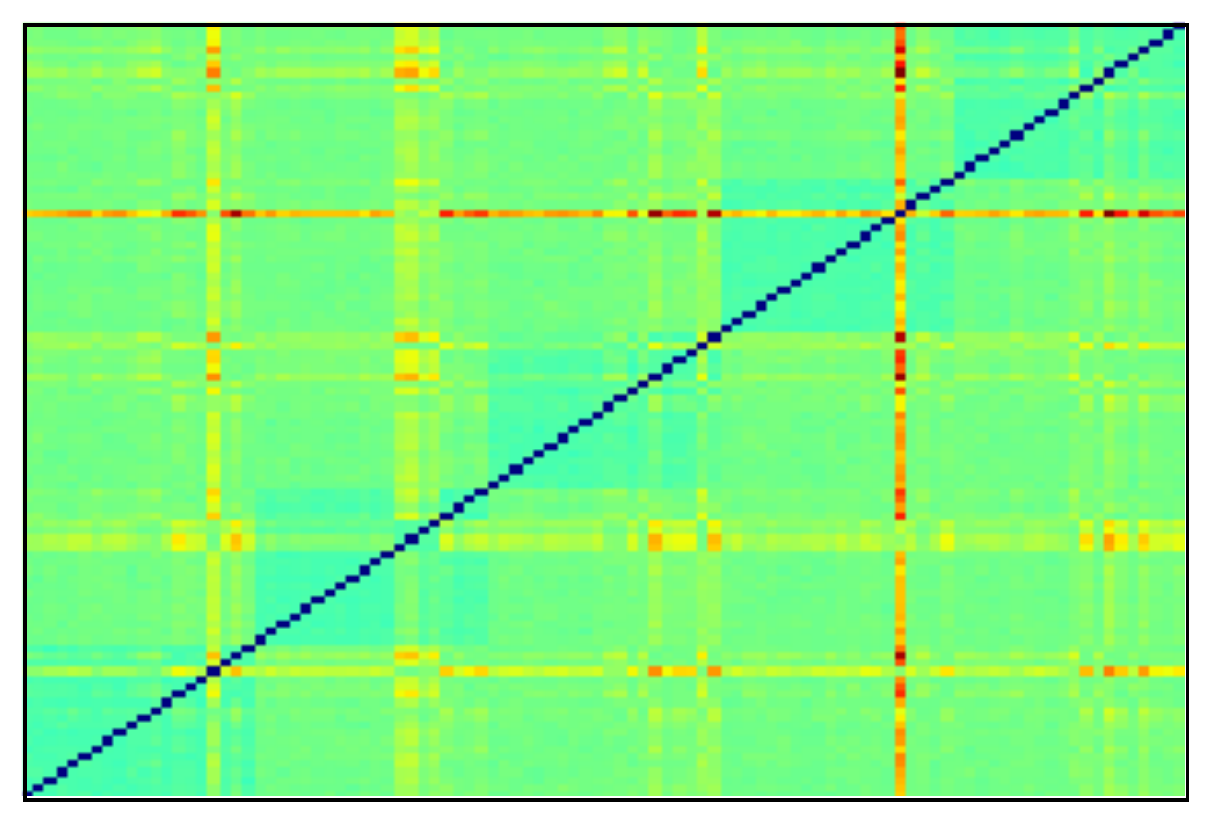}
GPR $\theta = 0.5$
\end{minipage}

\begin{minipage}[b]{0.325\linewidth}
\centering
\includegraphics[width=\textwidth]{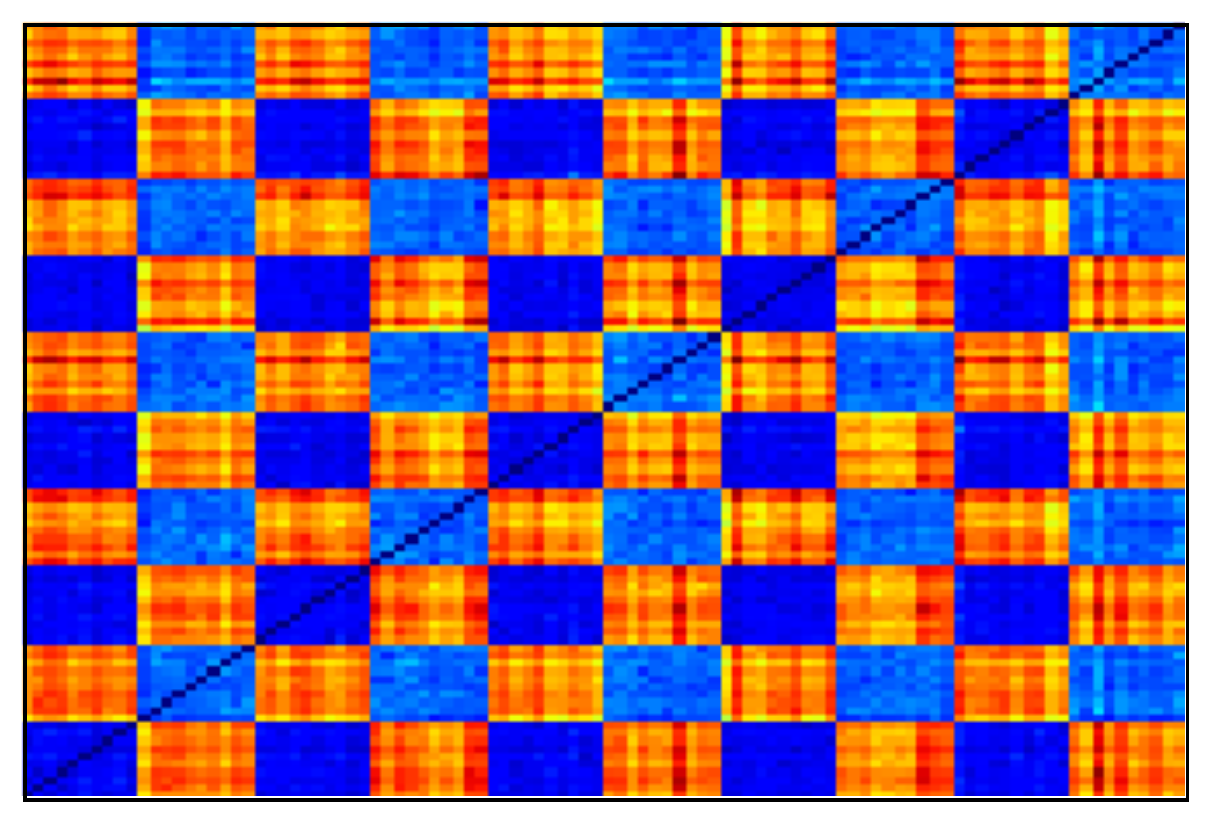}
GNPR $\theta = 0$
\end{minipage}
\begin{minipage}[b]{0.325\linewidth}
\centering
\includegraphics[width=\textwidth]{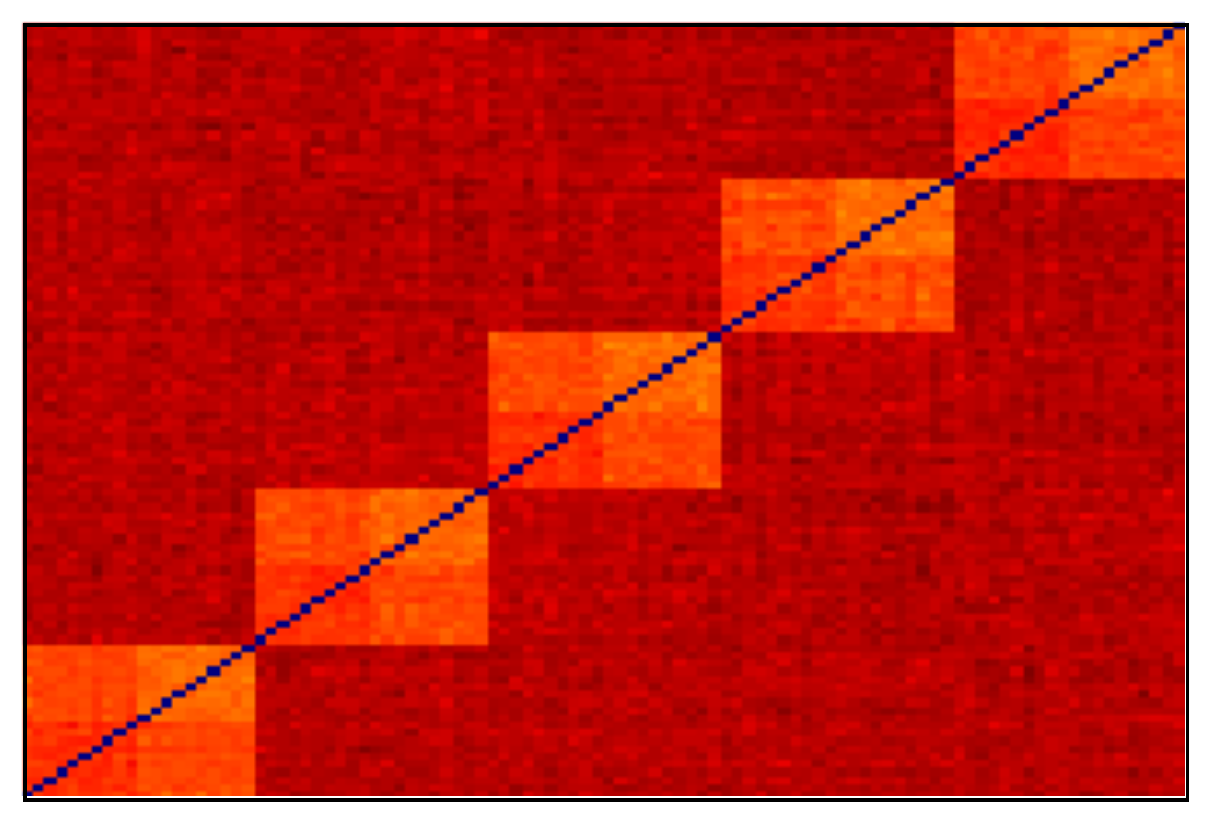}
GNPR $\theta = 1$
\end{minipage}
\begin{minipage}[b]{0.325\linewidth}
\centering
\includegraphics[width=\textwidth]{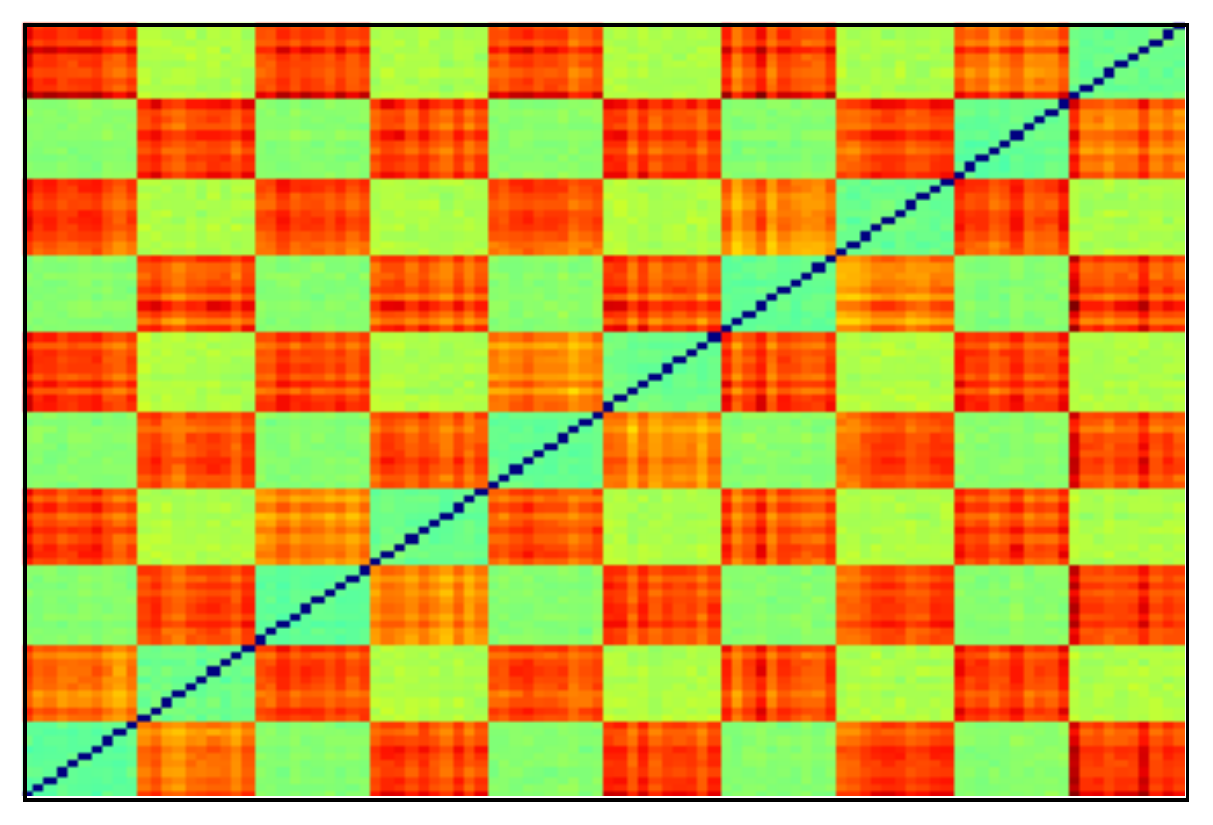}
GNPR $\theta = 0.5$
\end{minipage}
\caption{GPR and GNPR distance matrices. Both GPR and GNPR highlight the 5 correlation clusters ($\theta = 1$), but only GNPR finds the 2 distributions ($\mathcal{S}$ and $\mathcal{L}$) subdividing them ($\theta = 0$). Finally, by combining both information GNPR ($\theta = 0.5$) can highlight the 10 original clusters, while GPR ($\theta = 0.5$) simply adds noise on the correlation distance matrix it recovers.}\label{fig:dist_comp_GPR_GNPR}
\end{figure}

\begin{table*}[!t]
\caption{\label{tab2}Model parameters for some interesting test case datasets}
\centering
\begin{tiny}
\begin{tabular}{|c|c|c|c|c|c|c|c||c|c|c|c|}
\hline
Clustering & Dataset & $N$ & $T$ & $Q$ & $K$ & $\beta$ & $Y_k$ & $Z_1^i$ & $Z_2^i$ & $Z_3^i$ & $Z_4^i$ \\
\hline
\multirow{1}{*}{Distribution}
 &
A & 
200 & 5000 & 4 & 1 & 0 & $\mathcal{N}(0,1)$ & $\mathcal{N}(0,1)$ & $\mathcal{L}$ & $\mathcal{S}$ & $\mathcal{N}(0,2)$
\\
\hline
\multirow{1}{*}{Dependence}
 &
B					  & 
						 200 & 5000 & 10 & 10 & 0.1 & $\mathcal{S}$ & $\mathcal{S}$ & $\mathcal{S}$ & $\mathcal{S}$ & $\mathcal{S}$
\\
\hline
\multirow{2}{*}{Mix}
              & 
C                      & 
                        200 & 5000 & 10 & 5 & 0.1 & $\mathcal{N}(0,1)$ & $\mathcal{N}(0,1)$ & $\mathcal{S}$ & $\mathcal{N}(0,1)$ & $\mathcal{S}$
\\
\cline{2-12}
              & 
G                     & 
                        $32,\ldots,640$ & $10,\ldots,2000$ & 32 & 8 & 0.1 & $\mathcal{N}(0,1)$ & $\mathcal{N}(0,1)$ & $\mathcal{N}(0,2)$ & $\mathcal{L}$ & $\mathcal{S}$
\\
\hline
\multicolumn{9}{@{}l}{}
\end{tabular}
\end{tiny}
\end{table*}

\subsection{Performance of clustering using GNPR}

We empirically show that the GNPR approach achieves better results than others using common distances regardless of the algorithm used on the defined test cases A, B and C described in Table \ref{tab2}. Test case A illustrates datasets containing only distribution information: there are 4 clusters of distributions. Test case B illustrates datasets containing only dependence information: there are 10 clusters of correlation between random variables which are heavy-tailed. Test case C illustrates datasets containing both information: it consists in 10 clusters composed of 5 correlation clusters and each of them is divided into 2 distribution clusters. Using scikit-learn implementation \citep{scikit-learn}, we apply $3$ clustering algorithms with different paradigms: a hierarchical clustering using average linkage (HC-AL), $k$-means++ (KM++), and affinity propagation (AP). Experiment results are reported in Table \ref{tab1}. GNPR performance is due to its proper representation (cf. Fig.~\ref{fig:dist_comp}).
Finally, we have noticed increasing precision of clustering using GNPR as time $T$ grows to infinity, all other parameters being fixed. The number of time series $N$ seems rather uninformative as illustrated in Fig. \ref{fig:NT_conv} (left) which plots ARI \citep{hubert1985comparing} between computed clustering and ground-truth of dataset G as an heatmap for varying $N$ and $T$. Fig. \ref{fig:NT_conv} (right) shows the convergence to the true clustering as a function of $T$.

\begin{figure}[!t]
\begin{minipage}[b]{0.24\linewidth}
\centering
\includegraphics[width=\textwidth]{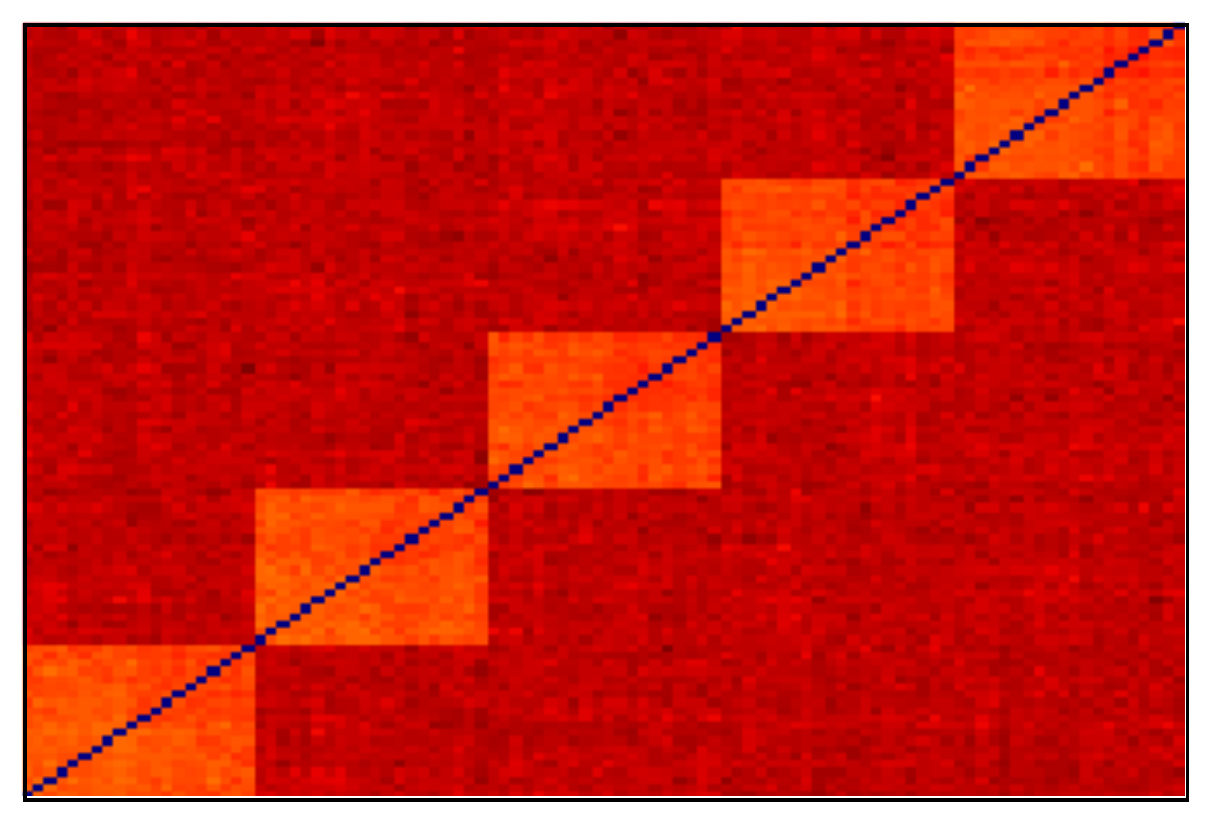}
$(1 - \rho) / 2$
\end{minipage}
\begin{minipage}[b]{0.24\linewidth}
\centering
\includegraphics[width=\textwidth]{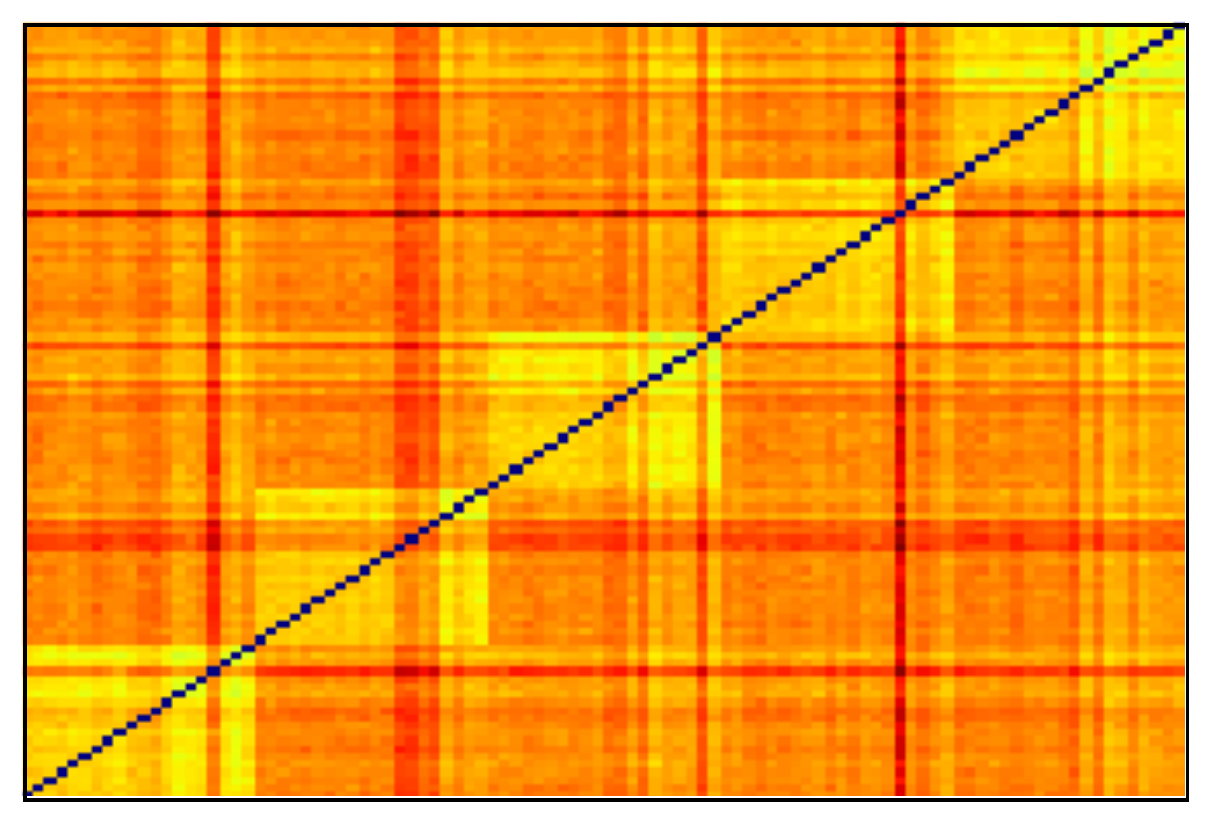}
$L_2$
\end{minipage}
\begin{minipage}[b]{0.24\linewidth}
\centering
\includegraphics[width=\textwidth]{GPR_mix_S}
GPR $\theta = 0.5$
\end{minipage}
\begin{minipage}[b]{0.24\linewidth}
\centering
\includegraphics[width=\textwidth]{GNPR_mix_S}
GNPR $\theta = 0.5$
\end{minipage}
\caption{Distance matrices obtained on dataset C using distance correlation, $L_2$ distance, GPR and GNPR. None but GNPR highlights the 10 original clusters which appear on its diagonal.}\label{fig:dist_comp}
\end{figure}
\begin{table}[!t]
\caption{\label{tab1}Comparison of distance correlation, $L_2$ distance, GPR and GNPR: GNPR approach improves clustering performance}
\centering
\begin{tabular}{|c|l|c|c|c|}
\cline{3-5}
\multicolumn{2}{ c }{} & \multicolumn{3}{ |c| }{Adjusted Rand Index}\\
\hline
Algo. & Distance & A & B & C\\
\hline
\hline
\multirow{8}{*}{HC-AL}    &
$(1-\rho)/2$ & 
 0.00 \tiny{$\pm 0.01$}   & 0.99 \tiny{$\pm 0.01$} & 0.56 \tiny{$\pm 0.01$}
\\
\cline{2-5}
 &
\small{$\mathbb{E}[(X - Y)^2]$}           & 
 0.00 \tiny{$\pm 0.00$}  & 0.09 \tiny{$\pm 0.12$} & 0.55 \tiny{$\pm 0.05$}
\\
\cline{2-5}
 &
\small{GPR} $\theta = 0$                   & 
 0.34 \tiny{$\pm 0.01$} & 0.01 \tiny{$\pm 0.01$} & 0.06 \tiny{$\pm 0.02$}
\\
\cline{2-5}
 &
\small{GPR} $\theta = 1$                   & 
 0.00 \tiny{$\pm 0.01$} & 0.99 \tiny{$\pm 0.01$} & 0.56 \tiny{$\pm 0.01$} 
\\
\cline{2-5}
 &
\small{GPR} $\theta = .5$                   & 
 0.34 \tiny{$\pm 0.01$} & 0.59 \tiny{$\pm 0.12$} & 0.57 \tiny{$\pm 0.01$}
\\
\cline{2-5} \rowcolor{LightCyan}
              & 
\small{GNPR} $\theta=0$                   &
 \textbf{1} & 0.00 \tiny{$\pm 0.00$} & 0.17 \tiny{$\pm 0.00$} 
\\
\cline{2-5} \rowcolor{LightCyan}
              & 
\small{GNPR} $\theta=1$                   &
  0.00 \tiny{$\pm 0.00$} & \textbf{1} & 0.57 \tiny{$\pm 0.00$}  
\\
\cline{2-5} \rowcolor{LightCyan}
              & 
\small{GNPR} $\theta=.5$                   &
 0.99 \tiny{$\pm 0.01$} & 0.25 \tiny{$\pm 0.20$} & \textbf{0.95 \tiny{$\pm 0.08$}}
\\
\hline
\hline
\multirow{8}{*}{KM++}    &
$(1-\rho)/2$ & 
  0.00 \tiny{$\pm 0.01$} & 0.60 \tiny{$\pm 0.20$} & 0.46 \tiny{$\pm 0.05$}
\\
\cline{2-5}
 &
\small{$\mathbb{E}[(X - Y)^2]$}           & 
  0.00 \tiny{$\pm 0.00$} & 0.34 \tiny{$\pm 0.11$} & 0.48 \tiny{$\pm 0.09$}
\\
\cline{2-5}
 &
\small{GPR} $\theta = 0$                   & 
 0.41 \tiny{$\pm 0.03$} & 0.01 \tiny{$\pm 0.01$} & 0.06 \tiny{$\pm 0.02$}
\\
\cline{2-5}
 &
\small{GPR} $\theta = 1$                   & 
 0.00 \tiny{$\pm 0.00$} & 0.45 \tiny{$\pm 0.11$} & 0.43 \tiny{$\pm 0.09$}
\\
\cline{2-5}
 &
\small{GPR} $\theta = .5$                   & 
 0.27 \tiny{$\pm 0.05$} & 0.51 \tiny{$\pm 0.14$} & 0.48 \tiny{$\pm 0.06$}
\\
\cline{2-5} \rowcolor{LightCyan}
              & 
\small{GNPR} $\theta=0$                   &
 \textbf{0.96 \tiny{$\pm 0.11$}} & 0.00 \tiny{$\pm 0.01$} & 0.14 \tiny{$\pm 0.02$}
\\
\cline{2-5} \rowcolor{LightCyan}
              & 
\small{GNPR} $\theta=1$                   &
 0.00 \tiny{$\pm 0.01$} & \textbf{0.65 \tiny{$\pm 0.13$}} & 0.53 \tiny{$\pm 0.02$}
\\
\cline{2-5} \rowcolor{LightCyan}
              & 
\small{GNPR} $\theta=.5$                   &
 0.72 \tiny{$\pm 0.13$} & 0.21 \tiny{$\pm 0.07$} & \textbf{0.64 \tiny{$\pm 0.10$}}
\\
\hline
\hline
\multirow{8}{*}{AP}    &
$(1-\rho)/2$ & 
0.00  \tiny{$\pm 0.00$}  & 0.99 \tiny{$\pm 0.07$} & 0.48 \tiny{$\pm 0.02$}
\\
\cline{2-5}
 &
\small{$\mathbb{E}[(X - Y)^2]$}           & 
0.14 \tiny{$\pm 0.03$} & 0.94 \tiny{$\pm 0.02$}  & 0.59 \tiny{$\pm 0.00$}
\\
\cline{2-5}
 &
\small{GPR} $\theta = 0$                   & 
0.25 \tiny{$\pm 0.08$} & 0.01 \tiny{$\pm 0.01$} & 0.05 \tiny{$\pm 0.02$}
\\
\cline{2-5}
 &
\small{GPR} $\theta = 1$                   & 
0.00 \tiny{$\pm 0.01$} & 0.99 \tiny{$\pm 0.01$}  & 0.48 \tiny{$\pm 0.02$}
\\
\cline{2-5}
 &
\small{GPR} $\theta = .5$                   & 
0.06 \tiny{$\pm 0.00$}  & 0.80 \tiny{$\pm 0.10$} & 0.52 \tiny{$\pm 0.02$}
\\
\cline{2-5} \rowcolor{LightCyan}
              & 
\small{GNPR} $\theta=0$                   &
\textbf{1} & 0.00 \tiny{$\pm 0.00$} & 0.18 \tiny{$\pm 0.01$}
\\
\cline{2-5} \rowcolor{LightCyan}
              & 
\small{GNPR} $\theta=1$                   &
0.00 \tiny{$\pm 0.01$} & \textbf{1} & 0.59 \tiny{$\pm 0.00$}
\\
\cline{2-5} \rowcolor{LightCyan} &
\small{GNPR} $\theta=.5$ &
0.39 \tiny{$\pm 0.02$} & 0.39 \tiny{$\pm 0.11$} & \textbf{1}
\\
\hline
\multicolumn{5}{@{}l}{}
\end{tabular}
\end{table}
\begin{figure}[!t]
\begin{minipage}[b]{0.35\linewidth}
\centering
\includegraphics[width=\textwidth]{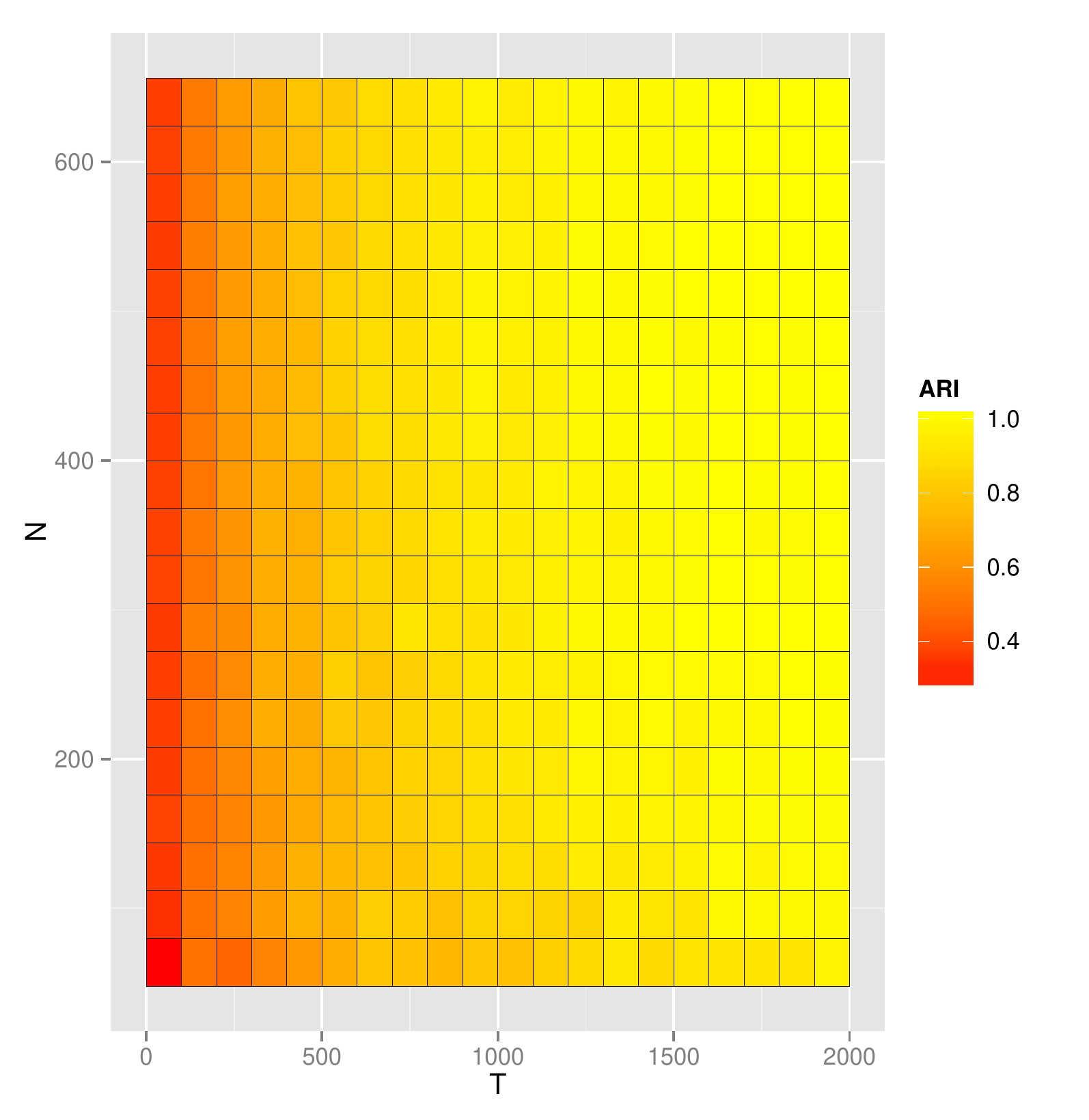}
\end{minipage}
\begin{minipage}[b]{0.65\linewidth}
\centering
\includegraphics[width=\textwidth]{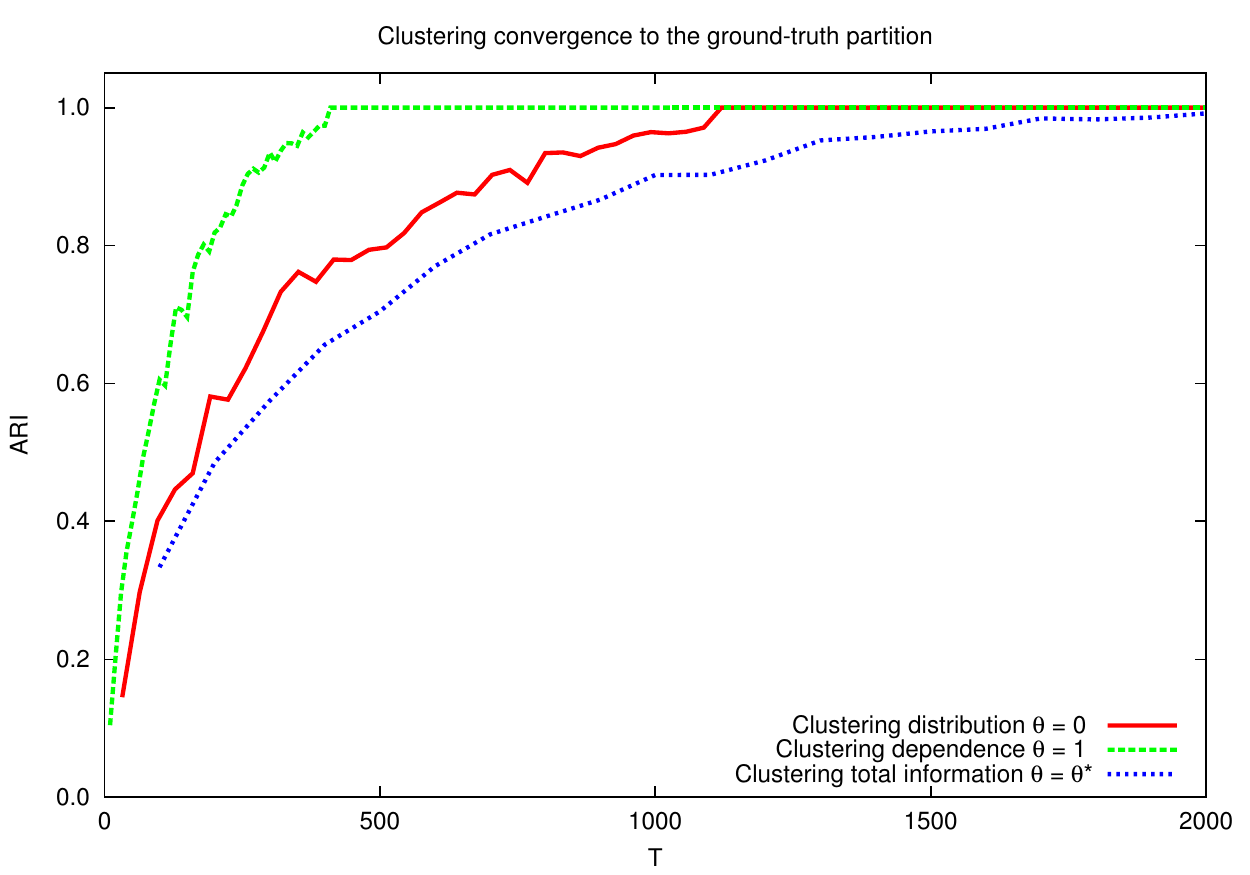}
\end{minipage}
\caption{Empirical consistency of clustering using GNPR as $T \rightarrow \infty$}\label{fig:NT_conv}
\end{figure}

\subsection{Application to financial time series clustering}

\subsubsection{Clustering assets: a (too) strong focus on correlation}

It has been noticed that straightfoward approaches automatically discover sector and industries \citep{gcalda:Mantegna99}. Since detected patterns are blatantly correlation-flavoured, it prompted econophysicists to focus on correlations, hierarchies and networks \citep{Tumminello2010} from the Minimum Spanning Tree and its associated clustering algorithm the Single Linkage to the state of the art \citep{musmeci2014relation} exploiting the topological properties of the Planar Maximally Filtered Graph \citep{lillo16} and its associated algorithm the Directed Bubble Hierarchical Tree (DBHT) technique \citep{citeulike:10441261}.
In practice, econophysicists consider the assets log returns and compute their correlation matrix. The correlation matrix is then filtered thanks to a clustering of the correlation-network \citep{citeulike:5519260} built using similarity and dissimilarity matrices which are derived from the correlation one by convenient \textit{ad hoc} transformations.
Clustering these correlation-based networks \citep{onnela2004clustering} aims at filtering the correlation matrix for standard portfolio optimization \citep{tola2008cluster}.
Yet, adopting similar approaches only allow to retrieve information given by assets co-movements and nothing about the specificities of their returns behaviour, whereas we may also want to distinguish assets by their returns distribution. For example, we are interested to know whether they undergo fat tails, and to which extent.

\subsubsection{Clustering credit default swaps}

We apply the GNPR approach on financial time series, namely daily credit default swap \citep{hull2006options} (CDS) prices. We consider the $N=500$ most actively traded CDS according to DTCC (\url{http://www.dtcc.com/}).
For each CDS, we have $T=2300$ observations corresponding to historical daily prices over the last 9 years, amounting for more than one million data points. Since credit default swaps are traded over-the-counter, closing time for fixing prices can be arbitrarily chosen, here 5pm GMT, i.e. after the London Stock Exchange trading session. This synchronous fixing of CDS prices avoids spurious correlations arising from different closing times. For example, the use of close-to-close stock prices artificially overestimates intra-market correlation and underestimates inter-market dependence since they have different trading hours \citep{martens2001returns}.
These CDS time series can be consulted on the web portal \url{http://www.datagrapple.com/}.

Assuming that CDS prices $(P^t)_{t \geq 1}$ follow random walks, their increments $\Delta P^t = P^{t} - P^{t-1}$ are i.i.d. random variables, and therefore the GNPR approach can be applied to the time series of prices variations, i.e. on data $(\Delta P_1^t,\ldots,\Delta P_N^t)$, $t=1,\ldots,T$. Thus, for aggregating CDS prices time series, we use a clustering algorithm (for instance, Ward's method \citep{Inchoate:Ward63}) based on the GNPR distance matrices between their variations.

Using GNPR $\theta = 0$, we look for distribution information in our CDS dataset. We observe that clustering based on the GNPR $\theta = 0$ distance matrix yields 4 clusters which fit precisely the multi-modal empirical distribution of standard deviations as can be seen in Fig.~\ref{fig:histo_vol}.
For GNPR $\theta = 1$, we display in Fig.~\ref{fig:correl_mat} the rank correlation distance matrix obtained. We can notice its hierarchical structure already described in many papers, e.g. \citep{gcalda:Mantegna99}, \citep{brida2010hierarchical}, focusing on stock markets.
There is information in distribution and in correlation, thus taking into account both information, i.e. using GNPR $\theta = 0.5$, should lead to a meaningful clustering. We verify this claim by using stability as a criterion for validation. Practically, we consider even and odd trading days and perform two independent clusterings, one on even days and the other one on odd days. We should obtain the same partitions. In Fig.~\ref{fig:stability_gnpr_cds}, we display the partitions obtained using the GNPR $\theta = 0.5$ approach next to the ones obtained by applying a $L_2$ distance on prices returns. We find that GNPR clustering is more stable than $L_2$ on returns clustering. Moreover, clusters obtained from GNPR are more homogeneous in size.

\begin{figure}[!t]
\includegraphics[width=\linewidth]{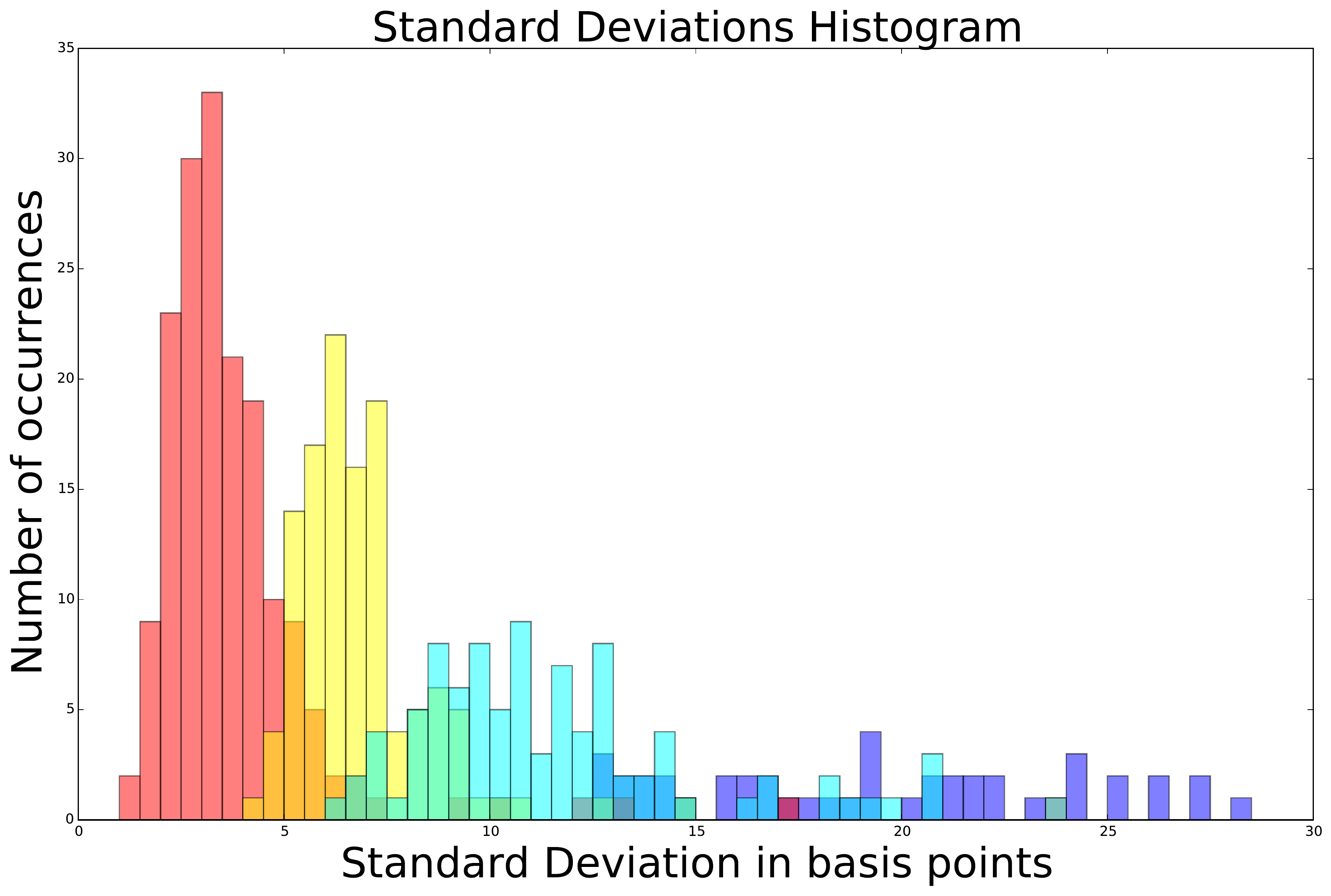}
\caption{Standard Deviation Histogram. The $4$ clusters found using GNPR $\theta = 0$ represented by the 4 colors fit precisely the multi-modal distribution of standard deviations.}\label{fig:histo_vol}
\end{figure}

\begin{figure}[!t]
\includegraphics[width=\linewidth]{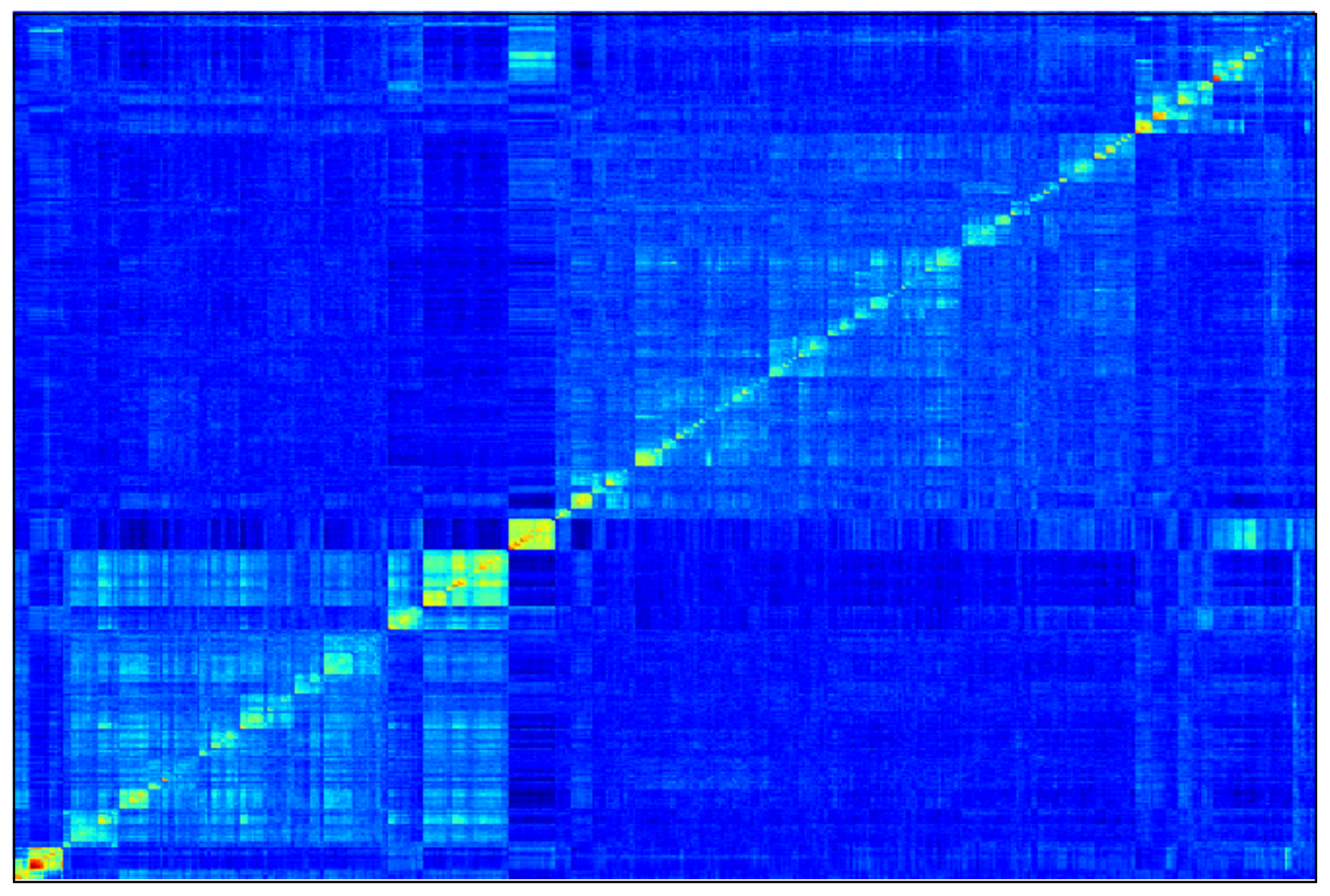}
\caption{Centered Rank Correlation Distance Matrix. GNPR $\theta = 1$ exhibits a hierarchical structure of correlations: first level consists in Europe, Japan and US; second level corresponds to credit quality (investment grade or high yield); third level to industrial sectors.}\label{fig:correl_mat}
\end{figure}

\begin{figure*}[!t]
\includegraphics[width=\linewidth]{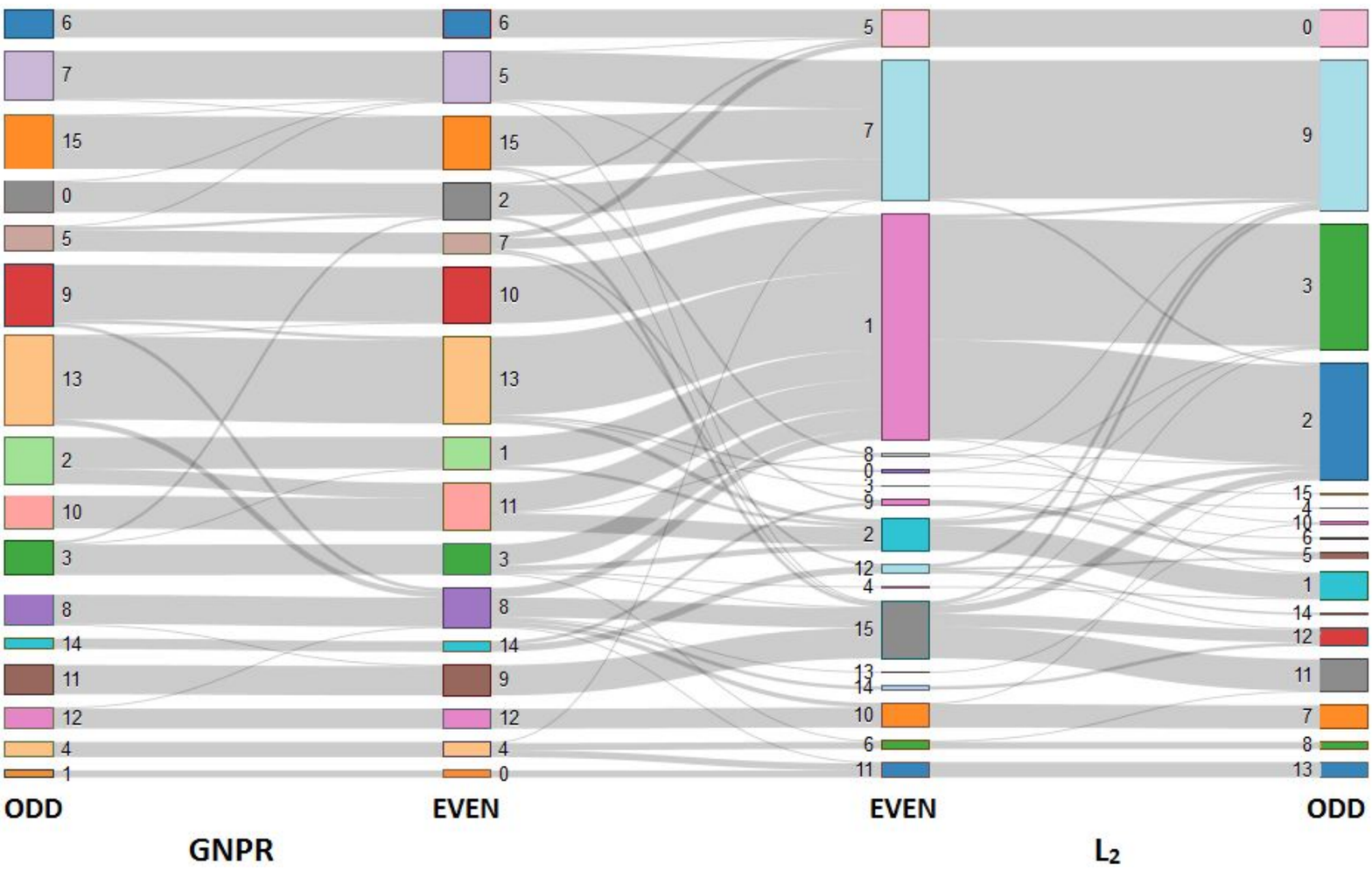}
\caption{Better clustering stability using the GNPR approach: GNPR $\theta = 0.5$ achieves ARI = 0.85; $L_2$ on returns achieves ARI 0.64; The two leftmost partitions built from GNPR on the odd/even trading days sampling look similar: only a few CDS are switching from clusters; The two rightmost partitions built using a $L_2$ on returns display very inhomogeneous (odd-2,3,9 vs. odd-4,14,15) and unstable (even-1 splitting into odd-3 and odd-2) clusters.}\label{fig:stability_gnpr_cds}
\end{figure*}

To conclude on the experiments, we have highlighted through clustering that the presented approach leveraging dependence and distribution information leads to better results: finer partitions on synthetic test cases and more stable partitions on financial time series.

\section{Discussion}\label{discus}

In this paper, we have exposed a novel representation of random variables which could lead to improvements in applying machine learning techniques on time series describing underlying i.i.d. stochastic processes.
We have empirically shown its relevance to deal with random walks and 
financial time series.
We have led a large scale experiment on the credit derivatives market notorious for not having Gaussian but heavy-tailed returns, first results are available on website \url{www.datagrapple.com}.
We also intend to lead such clustering experiments for testing applicability of the method to areas outside finance.
On the theoretical side, we plan to improve the aggregation of the correlation and distribution part by using elements of information geometry theory and to study the consistency property of our method.
\section*{Acknowledgements}

We thank Frank Nielsen, the anonymous reviewers, and our colleagues at Hellebore Capital Management for giving feedback and proofreading the article.

\bibliographystyle{model2-names}

\end{document}